\documentclass{aamas2016}

\usepackage{amsmath, amssymb}
\usepackage{algorithm, algorithmicx,algpseudocode}
\usepackage{color}
\usepackage{graphicx, subcaption}
\usepackage{comment}
\newtheorem{thm}{Theorem}
\newtheorem{lem}{Lemma}

\newtheorem{defn}{Definition}
\newcommand{\kareem}[1]{\textcolor{blue}{KA: \emph{#1}}}

\newcommand{\citep}[1]{\cite{#1}}

\newcommand{\citet}[1]{\cite{#1}}
\newcommand{\citey}[1]{(\citeyear{#1})}

\begin{document}

\title{Towards Resolving Unidentifiability in Inverse Reinforcement Learning}

\numberofauthors{3}

\author{ Kareem Amin \\ University of Michigan \\ amkareem@umich.edu  \alignauthor Satinder Singh \\ University of Michigan \\ baveja@umich.edu }

\maketitle

\begin{abstract}
We consider a setting for Inverse Reinforcement Learning (IRL) where the learner is extended with the ability to \emph{actively} select multiple environments, observing an agent's behavior on each environment. We first demonstrate that if the learner can experiment with \emph{any} transition dynamic on some fixed set of states and actions, then there exists an algorithm that reconstructs the agent's reward function to the fullest extent theoretically possible, and that requires only a small (logarithmic) number of experiments. We contrast this result to what is known about IRL in single fixed environments, namely that the true reward function is fundamentally unidentifiable. We then extend this setting to the more realistic case where the learner may not select any transition dynamic, but rather is restricted to some fixed set of environments that it may try. We connect the problem of maximizing the information derived from experiments to active submodular function maximization, and demonstrate that a greedy algorithm is near optimal (up to logarithmic factors). Finally, we empirically validate our algorithm on an environment inspired by behavioral psychology.  
\end{abstract}

\section{Introduction}

Inverse reinforcement learning (IRL), first introduced by Ng and Russell \citey{ng2000algorithms}, is concerned with the problem of inferring the (unknown) reward function of an agent behaving optimally in a Markov decision process. The most basic formulation of the problem asks: given a known environment\footnote{We will use the terminology environment to refer to an MDP without a reward function.} $E$, and an optimal agent policy $\pi$, can we deduce the reward function $R$ which makes $\pi$ optimal for the MDP $(E,R)$? 

IRL has seen a number of applications in the development of autonomous systems, such as autonomous vehicle operation, where even a cooperative (human) agent might have great difficultly describing her incentives \cite{smart2002effective,abbeel2004apprenticeship,abbeel2007application,coates2009apprenticeship}. However, the problem is fundamental to almost any study which involves behavioral modeling. Consider an experimental psychologist attempting to understand the internal motivations of a subject, say a mouse, or consider a marketer observing user behavior on a website, hoping to understand the potential consumer's value for various offers. 

As noted by Ng and Russell, a fundamental complication to the goals of IRL is the impossibility of identifying the exact reward function of the agent from its behavior. In general, there may be infinitely many reward functions consistent with any observed policy $\pi$ in some fixed environment. Since the true reward function is fundamentally unidentifiable, much of the previous work in IRL has been concerned with the development of heuristics which prefer certain rewards as better explanations for behavior than others \cite{ng2000algorithms,ziebart2008maximum,ramachandran2007bayesian}.
In contrast, we make several major contributions towards directly resolving the issue of unidentifiability in IRL in this paper. 

As a first contribution, we separate the causes of this unidentifiability into three classes. 1) A trivial reward function, assigning constant reward to all state-action pairs, makes all behaviors optimal; the agent with constant reward can execute any policy, including the observed $\pi$. 2) Any reward function is behaviorally invariant under certain arithmetic operations, such as re-scaling. Finally, 3) the behavior expressed by some observed policy $\pi$ may not be sufficient to distinguish between two possible reward functions both of which \emph{rationalize the observed behavior}, i.e., the observed behavior could be optimal under both reward functions. We will refer to the first two cases of unidentifiability as \emph{representational unidentifiability}, and the third as \emph{experimental unidentifiability}. 

As a second contribution, we will demonstrate that, while representational unidentifiability is unavoidable, experimental unidentifiability is not. In contrast to previous methods, we will demonstrate how the latter can be eliminated completely in some cases. Moreover, in a manner which we will make more precise in Section \ref{sec:identification}, we will argue that in some ways representational unidentifiability is superficial; by eliminating experimental unidentifiability, one arrives at the fullest possible  characterization of an agent's reward function that one can hope for.

As a third contribution, we develop a slightly richer model for IRL. We will suppose that the learner can observe the agent behaving optimally in \emph{a number of environments of the learner's choosing}. Notice that in many of our motivating examples it is reasonable to assume that the learner does indeed have this power. One can ask the operator of a vehicle to drive through multiple terrains, while the experimental psychologist might observe a mouse across a number of environments. It is up to the experimenter to organize the dynamics of the maze. One of our key results will be that, with the right choice of environments, the learner can eliminate experimental unidentifiability. We will study our {\bf repeated experimentation for IRL} in two settings, one in which the learner is omnipotent in that there are no restrictions on what environments can be presented to the agent, and another in which there are restrictions on the type of environments the learner can present. We show that in the former case, experimental unidentifiability can be eliminated with just a small number of environments. In the latter case, we cast the problem as budgeted exploration, and show that for some number of environments $B$, a simple greedy algorithm approximately maximizes the information revealed about $R$ in $B$ environments.

\paragraph{Most Closely Related Work}

Prior work in IRL has mostly focused on inferring an agent's reward function from data acquired from a fixed environment \cite{ng2000algorithms,abbeel2004apprenticeship,coates2008learning,ziebart2008maximum,ramachandran2007bayesian,syed2007game,regan2010robust}. We consider a setting in which the learner can actively select multiple environments to explore, before using the observations obtained from these environments to infer an agent's reward. Studying a model where the agent can make active selections of environments in an IRL setting is novel to the best of our knowledge. Previous applications of active learning to IRL have considered settings where, \emph{in a single environment}, the learner can query the agent for its action in some state \cite{lopes2009active}, or for information about its reward \cite{regan2009regret}. 

There is prior work on using data collected from multiple --- but exogenously fixed --- environments to predict agent behavior \cite{ratliff2006maximum}. There are also applications where methods for single-environment MDPs have been adapted to multiple environments \cite{ziebart2008maximum}. Nevertheless, both these works do not attempt to resolve the ambiguity inherent in recovering the true reward in IRL, and describe IRL as being an ``ill-posed'' problem. As a result these works ultimately consider the objective of mimicking or predicting an agent's optimal behavior. While this is a perfectly reasonable objective, we will more be interested in settings where the identification of $R$ is the goal in itself. Among many other reasons, this may be because the learner explicitly desires an interpretable model of the agent's behavior, or because the learner desires to transfer the learned reward function to new settings. 

In the economics literature, the problem of inferring an agent's utility from behavior has long been studied under the heading of utility or preference elicitation \cite{chajewska2000making,von2007theory, regan2011eliciting,rothkopf2011preference,regan2009regret,regan2011eliciting}. When these models analyze Markovian environments, they will assume a fixed environment where the learner can ask certain types of queries, such as bound queries eliciting whether some state-action reward $r(s,a) \geq b$. We will instead be interested in cases where the learner can only make inferences from agent behavior (with no external source of information), but can manipulate the environments on which the agent acts. 

\section{Setting and Preliminaries}\label{sec:prelim}

\begin{comment}
We consider a setting in which a learner can observe an agent behaving optimally with respect to some environment. In contrast to single-environment IRL, we imagine that the agent can repeatedly, and adaptively, select new environments for the agent. 
\end{comment}

We denote an environment by a tuple $E = (S,A,P,\gamma)$, where $S = \{1,...,d\}$ is a finite set of states in which the agent can find itself, $A$ is a finite set of actions available to the agent, and $P$ is a collection of transition dynamics for each $a \in A$, so that $P = \{P_a\}_{a \in A}$. We represent each $P_a$ as a row-stochastic matrix, with $P_a \in \mathbb{R}^{d \times d}$, and $P_a(s,s')$ denoting the agent's probability of transitioning to state $s'$ from state $s$ when selecting action $a$. The agent's discount factor is $\gamma \in (0,1)$.

We represent an agent's reward function as a vector $R \in \mathbb{R}^d$ with $R(s)$ indicating the (undiscounted) payout for arriving at state $s$. Note that a joint choice of Markovian environment $E$ with reward function $R$ fixes an MDP $M = (E,R)$. A policy is a mapping $\pi : S \rightarrow A$. With slight abuse of notation, we can represent $\pi$ as a matrix $P_{\pi}$ where $P_{\pi}(s,\cdot) = P_{\pi(s)}(s,\cdot)$ (we take the $s$-row of $P_{\pi}$ to be the $s$-row of $P_a$, where $a$ is the action chosen in state $s$).

Let $\mathrm{OPT}(E,R)$ denote the set of policies that are optimal, maximizing the agent's expected time-discounted rewards, for the MDP $(E,R)$. We consider a {\bf repeated experimentation} setting, where we suppose that the learner is able to select a sequence of environments\footnote{Defined on the same state and action spaces.} $E^1,E^2,...$, sequentially observing $\pi^1,\pi^2,...$ satisfying $\pi^i \in \mathrm{OPT}(E^i, R)$, for some unknown agent reward function $R$. We call each $(E^i, \pi^i)$ an \emph{experiment}. The goal of the experimenter is to output a reward estimate $\hat{R}$, approximating the true reward function. In many settings, the assumption that the learner can directly observe the agent's full policy $\pi^i$ is too strong, and a more realistic assumption is the learner observes only  \emph{trajectories} $T^i$, where $T^i$ denotes a sequence of state-action, pairs drawn according to the distribution induced by the agent playing policy $\pi^i$ in environment $E^i$. We will refer to the former feedback model as the {\bf policy observation setting}, and the latter as the {\bf trajectory observation setting}.

A fundamental theorem for IRL follows from rewriting the Bellman equations associated with the optimal policy in a single MDP, noting that the components of the vector $P_a (I - \gamma P_{\pi})^{-1}R$ correspond to the Q-value for action $a$, under policy $\pi$ and reward $R$, for each of $d$ states. 

\begin{thm}[Ng, Russell \cite{ng2000algorithms}]\label{thm:ngrussel}
Let $E = (S, A, P, \gamma)$ be an arbitrary environment, and $R \in \mathbb{R}^d$. $\pi \in \mathrm{OPT}(E,R)$ if and only if $\forall a \in A$, $(P_\pi - P_a)(I - \gamma P_{\pi})^{-1}R \geq 0$.\footnote{The inequality is read component-wise. That is, the relation holds if standard $\geq$ holds for each component.}
\end{thm}

The key take-away from this theorem is that in a policy observation setting, the set of reward functions $R$ consistent with some observed optimal policy $\pi$ are precisely those satisfying some set of linear constraints. Furthermore, those constraints can be computed from the environment $E$ and policy $\pi$. Thus, an object that we will make recurring reference to is the set of reward functions consistent with experiment $(E, \pi)$, denoted $K(E,\pi)$:

\begin{align*}
K(E, \pi) = \{ R \in \mathbb{R}^d \mid &\forall a \in A, (P_\pi - P_a)(I - \gamma P_{\pi})^{-1}R \geq 0 \\
&\forall s \in \mathcal{S}, R_{\min} \leq R(s) \leq R_{\max}\}.
\end{align*}

Since $K(E,\pi)$ is an intersection of linear constraints, it defines a convex polytope, a fact which will be of later algorithmic importance. An immediate corollary of Theorem \ref{thm:ngrussel}, is that given a sequence of experiments $\mathcal{E} = ((E^1,\pi^1), ..., (E^n, \pi^n))$, the set of rewards consistent with $\mathcal{E}$ are precisely those in 

$$K(\mathcal{E}) \triangleq \cap_{(E, \pi) \in \mathcal{E}} K(E, \pi)$$ 

We can also think of a trajectory $T$ as inducing a \emph{partial} policy $\pi_T$ on the states visited by the trajectory. In particular, let $\mathcal{D}(T)$ denote the domain of $T$, $\mathcal{D}(T) = \{s \mid \exists (s, a) \in T\}$. We say two policies $\pi,\pi'$ are consistent on $\mathcal{D}(T) \subset S$, denoted $\pi \equiv_{\mathcal{D}(T)} \pi'$, iff $\pi(s) = \pi'(s)$ for all $s \in \mathcal{D}(T)$. Thus, given $(E,T)$, the set of rewards consistent with the observation are precisely $K(E,T) = \{R \in \mathbb{R}^d \mid \exists \pi \equiv_{\mathcal{D}(T)} \pi_T, \forall a \in A, (P_\pi - P_a)(I - \gamma P_{\pi})^{-1}R \geq 0, R_{\min} \leq R_i \leq R_{\max}\}$, and given a sequence $\mathcal{E} = \{(E^1, T^1),...,(E^n,T^n)\}$, we can define $K(\mathcal{E})$ in the trajectory setting. 

\section{On Identification}\label{sec:identification} 
\begin{comment}
One goal of our work is to demonstrate that it is indeed possible to reconstruct the agent's true reward function with repeated experimentation. Before detailing an algorithm, a careful reader might reason that Theorem \ref{thm:ngrussel} precludes the possibility of such a result. In particular, no single experiment can ever determine that the agent's reward function is not a constant reward function (e.g., the all zeros vector). If we consider an agent whose true reward is some $R \in \mathbb{R}^d, R \not= \vec{0}$, then \emph{even in the policy oracle setting}, a single experiment $(E, \pi)$ cannot distinguish $R$ from $\vec{0}$. That is, both $R, \vec{0} \in K(E, \pi)$. Furthermore, this persists over any sequence of experiments $\mathcal{E}$, where both $R, \vec{0} \in K(\mathcal{E})$.  
\end{comment}

In this section we will give a more nuanced characterization of what it means to identify a reward function. We will argue that there are multiple types of uncertainty involved in identifying $R$, which we categorize as \emph{representational unidentifiability} and \emph{experimental unidentifiability}. Furthermore, we argue that first type is in some ways superficial, and ought to be ignored, while the second type can be eliminated.

We begin with a definition. Let $R$ and $R'$ be reward functions defined on the same state space $\mathcal{S}$. We say that $R$ and $R'$ are \emph{behaviorally equivalent} if for \emph{any} environment (also defined on $\mathcal{S}$), the agent whose reward function is $R$ behaves identically to the agent whose reward function is $R'$. 

\begin{defn}\label{defn:equiv}
Two reward vectors $R,R' \in \mathbb{R}^d$ defined on $\mathcal{S}$ are \emph{behaviorally equivalent}, denoted $R \equiv R'$ if for any set of actions, transition dynamics, and discount, $(\mathcal{A},\mathcal{P},\gamma)$, defining an environment $E = (\mathcal{S},\mathcal{A},\mathcal{P},\gamma)$ we have that \\
$\mathrm{OPT}(E, R) = \mathrm{OPT}(E,R')$.
\end{defn}

Behavioral equivalence defines an equivalence relation over vectors in $\mathbb{R}^d$, and we let $[R] = \{R' \in \mathbb{R}^d \mid R' \equiv R\}$ denote the equivalence classes defined in this manner. Intuitively, if $R$ and $R'$ are behaviorally equivalent, they induce identical optimal policies in every single environment, and therefore are not really ``different'' reward functions. They are simply different representations of the same incentives. 

We now observe that behavioral equivalence classes are invariant under multiplicative scaling by positive scalars, and component-wise translation by a constant. Intuitively, this is easy to see. Adding $c$ reward to every state in some reward function $R$ does not affect an agent's decision-making. This is simply ``background'' reward that the agent gets for free. Similarly, scaling $R$ by a positive constant simply changes the ``units" used to represent rewards.  The agent does not, and should not, care whether its reward is represented in dollars or cents. We prove this formally in the following Theorem.

\begin{thm}\label{thm:equiv}
For any $c \in \mathbb{R}^d$, let $\vec{c} \in \mathbb{R}^d$ denote the vector with all components equal to $c$. For any $\alpha > 0$, and $R \in \mathbb{R}^d$, $R \equiv \alpha R + \vec{c}$. 
\end{thm}
\begin{proof}
First consider $\vec{c}$ as defined in the statement of the Theorem. Fix any environment $E = (\mathcal{S},\mathcal{A}, \mathcal{P}, \gamma)$, action $a \in \mathcal{A}$ and arbitrary policy $\pi$. We begin by claiming that $(P_\pi - P_a)(I - \gamma P_{\pi})^{-1}\vec{c} = \vec{0}$. 

The Woodbury formula for matrix inversion tells us that $(I - \gamma P_{\pi})^{-1} = I + (I - \gamma P_{\pi})^{-1} \gamma P_{\pi}$. Furthermore, for any row-stochastic matrix $P$, $P\vec{c} = \vec{c}$. Therefore:

\begin{align*}
v &= (P_\pi - P_a)(I - \gamma P_{\pi})^{-1}\vec{c}\\ &= (P_\pi - P_a)(I + (I - \gamma P_{\pi})^{-1}\gamma P_{\pi})\vec{c}\\
&= (P_\pi - P_a)\vec{c} + (P_{\pi} - P_a)(I - \gamma P_{\pi})^{-1}\gamma P_{\pi}\vec{c}\\
&= \vec{0} + (P_{\pi} - P_a)(I - \gamma P_{\pi})^{-1}\gamma \vec{c} = \gamma v
\end{align*}

Since $\gamma \in (0,1)$, it must be that $v = \vec{0}$. 

Now fix a reward function $R \in \mathbb{R}^d$, and arbitrary environment $E$, and consider $\mathrm{OPT}(E,R)$. By Theorem \ref{thm:ngrussel}, we know that $\pi \in \mathrm{OPT}(E,R)$ iff for any $a \in A$, $(P_\pi - P_a)(I - \gamma P_{\pi})^{-1}R \geq 0$, which occurs iff  $(P_\pi - P_a)(I - \gamma P_{\pi})^{-1}( \alpha R) \geq 0$, since $\alpha$ is a positive scalar. Finally, we can conclude that $\pi \in \mathrm{OPT}(E,R)$ iff for all $a \in \mathcal{A}$, $(P_\pi - P_a)(I - \gamma P_{\pi})^{-1}( \alpha R + \vec{c}) \geq 0$, this last condition implying that $\pi \in \mathrm{OPT}(E,\alpha R + \vec{c})$, again by Theorem \ref{thm:ngrussel}. 

Since our choice of $E$ was arbitrary, by Definition \ref{defn:equiv}, $R \equiv \alpha R + \vec{c}$, concluding the proof.
\end{proof}

Thus, we argue that one reason why reward functions cannot be identified is a trivial one: the classic IRL problem does not fix a consistent representation for reward functions.  For any $R \in \mathbb{R}^d$ there are an uncountable number of other functions in $[R]$, namely $\alpha R + \vec{c}$ for any $\alpha$ and $\vec{c}$, all of which are behaviorally identical to $R$. However, distinguishing between these functions is irrelevant; whether an agent's ``true'' reward function is $(1,2,3,4)$ or $(0,1/3,2/3,1)$\footnote{We get $(0,1/3,2/3,1)$ from $(1,2,3,4)$ by subtracting $1$ from every state and dividing by $3$} is simply a matter of what units are used to represent rewards. 

In light of this observation, it is convenient to fix a \emph{canonical element} of each equivalence class $[R]$. For any constant reward function $R$, we will take its canonicalized representation to be $\vec{0}$. Otherwise we note, by way of Theorem \ref{thm:equiv}, that any $R$ can be translated and re-scaled so that $\max_s R(s) = 1$ and $\min_s R(s) = 0$. More carefully, for any non-constant $R$, we take its canonicalized representation to be $(R - \min_s R(s))/(\max_s R(s) - \min_s R(s)) \in [R]$. This canonicalization is consistent with behavioral equivalence, and we state the following Theorem whose proof can be found in the appendix. As a consequence of this Theorem, we can use the notation $[R]$ interchangeably to refer to the equivalence class of $R$, or the the unique canonical element of $[R]$.

\begin{thm}
For any $R,R' \in \mathbb{R}^d$, $R \equiv R'$ if and only if they have the same canonicalized representation. 
\end{thm}

We next consider the issue of trivial/constant rewards $[\vec{0}]$. Since the IRL problem was first formulated, it has been observed that no single experiment can ever determine that the agent's reward function is not a constant reward function. The algebraic reason for this is the fact that $\vec{0}$ is always a solution to the linear system $K(E,\pi)$, for any $E$ and $\pi$. The intuitive reason for this is the fact that any $\pi$ on some $E$ is as optimal as any other policy for an agent whose reward is $\vec{0}$. Therefore, if we consider an agent whose true reward is some $R \in \mathbb{R}^d, R \not= \vec{0}$, then \emph{even in the policy observation setting}, both $R, \vec{0} \in K(E, \pi)$.  Furthermore, this will not disappear with multiple experimentation. After any sequence of experiments $\mathcal{E}$, it also remains that both $R, \vec{0} \in K(\mathcal{E})$. 

Consider an agent whose true reward function is $\vec{0}$. A crucial consequence of the above is that if an IRL algorithm guarantees that it will identify $\vec{0}$, then it necessarily misidentifies non-trivial reward functions. This is because an agent with a trivial reward function is allowed to behave arbitrarily, and therefore may choose to behave consistently with some non-trivial reward $R$. An IRL algorithm that guarantees identification of trivial rewards will therefore misidentify the agent whose true reward is $R$. 

This leads us to the following revised definition of identification, which  accounts for what we call representational unidentifiability:

\begin{defn}\label{defn:ident}
We say that an IRL algorithm succeeds at identification if for any $R \in \mathbb{R}^d$, after observing behavior from an agent with true reward $R$, the algorithm outputs a $\hat{R}$ such that $\hat{R} \equiv R$ whenever $R \not\in [\vec{0}]$.
\end{defn}

Notice that this definition accomplishes two things. First, it excuses an algorithm for decisions about how $R$ is represented. In other words, it asserts that the salient task in IRL is computing a member of $[R]$, not the literal $R$. Secondly, if the true reward function $R$ is not constant (i.e. $R \not\in [\vec{0}]$), it demands the that algorithm identify $R$ (up to representational decisions). However, if the agent really does have a reward function of $\vec{0}$, the algorithm is allowed to output anything. In other words, the Algorithm is only allowed to behave arbitrarily if the agent behaves arbitrarily.\footnote{We comment that, as a practical matter, one is usually interested in rationalizing the behavior of an agent believed to be non-trivial.}
 
We also note that Definition \ref{defn:ident} can be relaxed to give a notion of approximate identification, which we state here:

\begin{defn}
We say that an IRL algorithm $\epsilon$-identifies a reward function if for any $R \in \mathbb{R}^d$, after observing behavior from an agent with true reward $R$, the algorithm outputs a $\hat{R}$ such that $||[R] - [\hat{R}]||_{\infty} \leq \epsilon$  whenever $R \not\in [\vec{0}]$. 
\end{defn}

Even Definition \ref{defn:ident} may not be attainable from a single experiment, as $K(E,\pi)$ may contain multiple behavioral classes $[R]$. We call this phenonmenon \emph{experimental unidentifiability}, due to the fact that the experiment $(E,\pi)$ may simply be insufficient to distinguish between some $[R]$ and $[R']$. In the next section, we will observe that this source of uncertainty in the reward function can be decreased with multiple experimentation, as depicted in Figure \ref{fig:cartoon} (see Caption for details). In other words, by distinguishing representational unidentifiability from experimental unidentifiability, we can formally resolve the latter.

\begin{figure}
\centering
    \begin{subfigure}[b]{0.12\textwidth}
  \includegraphics[width=\textwidth]{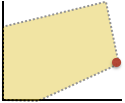}
  \caption{}
  \end{subfigure}
    \begin{subfigure}[b]{0.12\textwidth}
  \includegraphics[width=\textwidth]{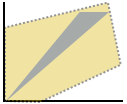}
  \caption{}
  \end{subfigure}
    \begin{subfigure}[b]{0.12\textwidth}
  \includegraphics[width=\textwidth]{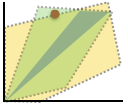}
  \caption{}
  \end{subfigure}
  \caption{{\small (a) After observing an agent's behavior in an environment, there is some set of rewards $K(E,\pi)$ consistent with the observed behavior, depicted by the shaded region. Previous work has been concerned with designing selection rules that pick some point in this region, depicted here by the red circle. (b) No amount of experimentation can remove the representational unidentifiability from the setting, depicted here by the darker shaded region. (c) Nevertheless, adding the constraints $K(E',\pi')$ induced by a second experiment disproves the original selection.}, removing some experimental unidentifiability.}
  \label{fig:cartoon}
\end{figure}

A more concrete example is given in Figure \ref{fig:mazes}, which depicts a grid-world with each square representing a state. In each of the figures, thick lines represent impenetrable walls, and an agent's policy is depicted by arrows, with a circle indicating the agent deciding to stay at a grid location. The goal of the learner is to infer the reward of each state. Figures \ref{fig:mazes}(a) and \ref{fig:mazes}(b), depict the same agent policy, which takes the shortest path to the location labeled $x$ from any starting location. One explanation for such behavior, depicted in Figure \ref{fig:mazes}(a), is that the agent has large reward for state $x$, and zero reward for every other state. However, an equally possible explanation is that the state $y$ also gives positive reward (but smaller than that of $x$) such that if there exists a shortest path to $x$ that also passes through $y$, the agent will take it (depicted in Figure \ref{fig:mazes}(b)). Without additional information, these two explanations cannot be distinguished. 

This is an example of experimental unidentifiability that can nevertheless be resolved with additional experimentation. By observing the same agent in the environment depicted in Figure \ref{fig:mazes}(c), the learner infers that $y$ is indeed a rewarding state. Finally, observing the agent's behavior in the environment of Figure \ref{fig:mazes}(d) reveals that the agent will prefer traveling to state $y$ if getting to $x$ requires 11 steps or more, while getting to $y$ requires 4 steps of fewer. These subsequent observations allow the learner to relate the agent's reward at state $x$ with the agent's reward at state $y$. 

\begin{figure}
\centering
    \begin{subfigure}[b]{0.33\textwidth}
  \includegraphics[width=\textwidth]{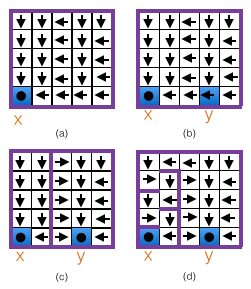}
  \end{subfigure}
  \caption{{\small (a) An agent's policy in a fixed environment. An agent can move in one of four directions or can stay at a location (represented by the black circle). Thick purple lines represent impassable walls. (d) An experiment revealing that if getting to $x$ requires 11 steps or more, and getting to $y$ requires 4 or fewer, the agent prefers $y$.}}
  \label{fig:mazes}
\end{figure}

\begin{comment}
As we will see, there are in fact two types of uncertainty in identifying $R$, illustrated in Figure \ref{fig:cartoon}. The first can be eliminated by conducting more experiments, and is possible in the repeated IRL setting. The second, which we call {\bf unit unidentifiability}, stems from the fact that rewards are invariant to scaling and translation. This second type of unidentifiability is in some ways superficial, and can be eliminated by fixing a unit system. There are many ways to do this, but if, for example, we commit to choices for $R_{\min}$ and $R_{\max}$ (say, $0$ and $1$), then we will see that this second type of unidentifiability vanishes. This is precisely how our first algorithm is able to accurately identify $R$. Each $\hat{R}(s)$ is written as an approximate convex combination of $R_{\min}$ and $R_{\max}$. Once the learner commits to these choices, then $||\hat{R} - R||_{\infty} \leq \epsilon$. We can think of $|R_{\max} - R_{\min}|$ as the the learner's prior belief about how arbitrarily the agent behaves. 
\end{comment}

\section{Omnipotent Experimenter\\ Setting}\label{sec:omni}

%As we alluded in the introduction, we will first 
We now consider a repeated experimentation setting in which the environments available for selection by the experimenter are completely unrestricted. 
%We do this primarily for pedagogical reasons, demonstrating that types of guarantees that can be made in a repeated experimentation setting can in fact be much stronger than those made in a standard single-environment IRL setting. 
Formally, 
%we suppose that 
each environment $E$ selected by the experimenter belongs to a class $\mathcal{\mathcal{U}}^*$ containing an environment $(S, A, P, \gamma)$ for every feasible set of transition dynamics $P$ on $S$. We call this the {\bf omnipotent experimenter setting}. 
%More generally, the experimenter is faced with a possibly constrained set of environments $\mathcal{U} \subset \mathcal{U}^*$, which we study in the {\bf restricted experimenter setting}. 

We will describe an algorithm for the omnipotent experimenter setting that $\epsilon$-identifies $R$, using just $O(\log(d/\epsilon))$ experiments. While the omnipotent experimenter is extremely powerful, the result demonstrates 
%what can be achieved in the repeated IRL setting. 
that the guarantee obtained in a repeated IRL setting can be far stronger than available in a standard single-environment IRL setting.
Furthermore, it clarifies the distinction between experimental unidentifiability and representational unidentifiability.

\subsection{Omnipotent Identification Algorithm}\label{sec:omnialg} The algorithm proceeds in two stages, both of which involve simple binary searches. The first stage will identify states $s_{\min}, s_{\max}$ such that $R(s_{\min}) = R_{\min}$ and $R(s_{\max}) = R_{\max}$. The second stage identifies for each $s \in S$ an $\alpha_s$ such that $R(s) = \alpha_s R_{\min} + (1-\alpha_s) R_{\max}$. Throughout, the algorithm only makes use of two agent actions which we will denote $a_1,a_2$. Therefore, in describing the algorithm, we will assume that $|A| = 2$, and the environment selected by the algorithm is fully determined by its choices for $P_{a_1}$ and $P_{a_2}$. If in fact $|A| > 2$, in the omnipotent experimenter setting, one can reduce to the two-action setting by making the remaining actions in $A$ equivalent to either $a_1$ or $a_2$.\footnote{Doing so is possible in this setting because transition dynamics can be set arbitrarily.} 

We first address the task of identifying $s_{\max}$. Suppose we have two candidates $s$ and $s'$ for $s_{\max}$. The key idea in this first stage of the algorithm is to give the agent an absolute choice between the two states by setting $P_{a_1}(s,s) = 1$, $P_{a_1}(s',s') = 1$, while setting $P_{a_2}(s,s') = 1$ and $P_{a_2}(s',s) = 1$. An agent selecting $\pi(s) = a_1$ reveals (for any $\gamma$) that $R(s) \geq R(s')$, while an agent selecting $\pi(s) = a_2$ reveals that $R(s) \leq R(s')$. This test can be conducted for up to $d/2$ distinct pairs of states in a single experiment. Thus given $k$ candidates for $s_{\max}$, in a single experiment, we can narrow the set of candidates to $k/2$, and are guaranteed that one of the remaining states $s$ satisfies $R(s) = R_{\max}$. After $\log(d)$ such experiments we can identify a single state $s_{\max}$ which satisfies $R(s_{\max}) \geq R(s)$ for all $s$. Conducting an analogous procedure identifies a state $s_{\min}$. 

Once $s_{\min}$ and $s_{\max}$ are identified, take $s_{1},...,s_{d-2}$ to be the remaining states, and consider an environment with transition dynamics parameterized by $\mathbf{\alpha} = (\alpha_{s_1},...,\alpha_{s_{d-2}})$. A typical environment in this phase is depicted in Figure \ref{fig:omni}. The environment sets $s_{\min}, s_{\max}$ to be sinks with $P_{a_1}(s_{\min},s_{\min}) = P_{a_1}(s_{\max},s_{\max}) = P_{a_2}(s_{\min},s_{\min}) = P_{a_2}(s_{\max}, s_{\max}) = 1$. For each remaining $s_i$, $P_{a_1}(s_i,s_{\min}) = \alpha_{s_i}$ and $P_{a_1}(s_i,s_{\max}) = (1-\alpha_{s_i})$, so that taking action $a_1$ in state $s_i$ represents an $\alpha_i$ probability gamble between the best and worst state. Finally, $P_{\mathbf{\alpha}}$ also sets $P_{a_2}(s,s) = 1$, and so taking action $a_2$ in state $s_i$ represents receiving $R(s_i)$ for sure. By selecting $\pi(s) = a_1$, the agent reveals $\alpha_s R_{\min} + (1-\alpha_s) R_{\max} \geq R(s)$, while a choice $\pi(s) = a_2$ reveals that $\alpha_s R_{\min} + (1-\alpha_s) R_{\max} \leq R(s)$. Thus, a binary search can be conducted on each $\alpha_s \in [0,1]$ independently in order to determine an $\epsilon$ approximation of the $\alpha^*_s$ such that $R(s) = \alpha_s^* R_{\min} + (1-\alpha^*_s)R_{\max}$. 

\begin{comment}
\begin{figure}[h]
  \centering
  \begin{minipage}[c]{0.2\textwidth}
    \includegraphics[width=\textwidth]{./omni_2.png}
  \end{minipage}
  \begin{minipage}[c]{0.25\textwidth}
  \caption{{\small Typical environment in the second phase of the algorithm. The dotted lines represent transitions for action $a_1$, while the solid lines represent transitions for action $a_2$.}}
  \label{fig:omni}
  \end{minipage}
\end{figure}
\end{comment}

\begin{figure}
        \centering
    \begin{subfigure}[b]{0.2\textwidth}
      \includegraphics[width=\textwidth]{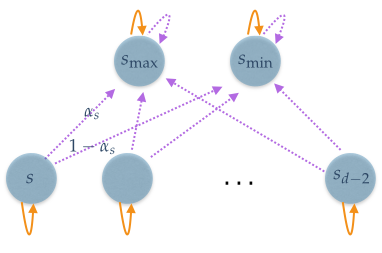}
    \end{subfigure}
  \caption{{\small Typical environment in the second phase of the algorithm. The dotted lines represent transitions for action $a_1$, while the solid lines represent transitions for action $a_2$.}}
\vspace{-1em}
  \label{fig:omni}
\end{figure}

The algorithm succeeds at $\epsilon$-identification, summarized in the following theorem. The proof of the theorem is a straightforward analysis of binary search.

\begin{thm}\label{thm:simple}
Let $\hat{R}$ be defined by letting $\hat{R}(s_{\min}) = 0$, $\hat{R}(s_{\max}) = 1$, and $\hat{R}(s) = 1 - \alpha_s$ for all other $s$ (where $s_{\min}$, $s_{\max}$, and $\alpha_s$ are identified as described above). For any true reward function $R \not\in [\vec{0}]$ with canonical form $[R]$, $||[R] - \hat{R}||_{\infty} \leq \epsilon$. 
\end{thm}

The takeaway of this setting is that the problems regarding identification in IRL can be circumvented with repeated experimentation. It is thought that even with policy observations, the IRL question is fundamentally ill-posed. However, here we see that with repeated experimentation it is in fact possible to identify $R$ to arbitrary precision in a well-defined sense. While these results are informative, we believe that it is unrealistic to imagine that the learner can arbitrarily influence the environment of the agent. In the next section, we develop a theory for repeated experimentation when the learner is restricted to select environments from some restricted subset of all possible transition dynamics. 

\section{Restricted Experimenter Setting}

We now consider a setting in which the experimenter has a restricted universe $\mathcal{U}$ of environments to choose from. $\mathcal{U}$ need not contain every possible transition dynamic, an assumption required to execute the binary search algorithm of the previous section. 
%As previously noted, 
The best the experimenter could ever hope for is to try every environment in $\mathcal{U}$. This gives the experimenter all the available information about the agent's reward function $R$. Thus, we will be more interested in maximizing the information gained by the experimenter while minimizing the number of experiments conducted. In practice, observing an agent may be expensive, or hard to come by, and so for even a small budget of experiments $B$, the learner would like select the environments from $\mathcal{U}$ which maximally reduce experimental unidentifiability.

Once a sequence of experiments $\mathcal{E}$ has been observed, we know that $R$ is consistent with the observed sequence if and only if $R \in K(\mathcal{E})$. Thus, the value of repeated experimentation is allowing the learner to select environments so that $K(\mathcal{E})$ is as informative as possible. In contrast, we note that previous work on IRL has largely been focused on designing heuristics for the selection problem of picking some $R$ from a fixed set (of equally possible reward functions). Thus, we will be interested in making $K(\mathcal{E})$ ``small,'' while IRL has traditionally been focused on selecting $R$ from exogenously fixed $K(\mathcal{E})$. Before defining what we mean by ``small'', we will review preexisting methods for selecting $R \in K(\mathcal{E})$.

\subsection{Generalized Selection Heuristics}
\label{sec:heuristics}

In the standard (single-environment) setting, given an environment $E$ and observed policy $\pi$, the learner must make a selection among one of the rewards in $K(E, \pi)$. The heuristic suggested by \citet{ng2000algorithms} is motivated by the idea that for a given state $s$, the reward function that maximizes the difference in Q-value between the observed action in state $s$, $\pi(s)$, and any other action $a \not= \pi(s)$, gives the strongest explanation of the behavior observed from the agent. Thus, a reasonable linear selection criterion is to maximize the sum of these differences across states. Adding a regularization term, encourages the selection of reward functions that are also sparse. Putting these together, the standard selection heuristic for single-environment IRL is to select the $R$ which maximizes:

\begin{equation}\label{eqn:classic}
{\small \sum_{s \in S} \left( \min_{a \not= \pi(s)} (P_\pi(s) - P_a(s))(I - \gamma P_{\pi})^{-1}R \right) - \lambda |R(s)|}
\end{equation}

There are two natural candidates for generalizing this selection rule to the repeated experimentation setting, where now instead of a single experiment, the experimenter has encountered a sequence of observations $\mathcal{E}$. The first is to \emph{sum over all (environment, state), pairs}, the minimum difference in Q-value between the action selected by the agent and any other action. The second is to sum over states, taking the \emph{minimum over all (environment, action), pairs}. While one could make arguments motivating each of these, ultimately any such objective is heuristic. However, we do argue that there is a strong algorithmic reason for preferring the latter objective. In particular, the former objective grows in dimensionality as environments are added, quickly resulting in an intractable LP. The dimension of the objective in the latter (Equation \ref{eq:rule}), however, remains constant.\footnote{Writing Equation \ref{eq:rule} as an LP in standard form requires translating the $\min$ into constraints, and thus the number of constraints grows with the number of experiments, but as we demonstrate in our experimental results, this is tractable for most LP solvers.}

{\small \begin{equation}\label{eq:rule}
\underset{R \in K(\mathcal{E})}{\mathrm{maximize}} \sum_{s \in S} \left( \min_{\underset{a \not= \pi^i(s)}{ (E^i, \pi^i) \in \mathcal{E}}} (P_{\pi}^i(s) - P_{a}^i(s))(I - \gamma P_{\pi}^i)^{-1}R \right) - \lambda |R(s)|
\end{equation}}

There are other selection rules for the single-environment setting, which are generalizable to the repeated experimentation setting, including heuristics for the infinite state setting, trajectory heuristics, as well as approaches already adapted to multiple environments \cite{ratliff2006maximum}. Due to space constraints, we discuss only the foundational approach of \citet{ng2000algorithms}. Our goal here is simply to emphasize the dichotomy between adapting pre-existing IRL methods to data gathered from multiple environments (however that data was generated), and the problem of how to best select those environments to begin with, this latter problem being the focus of the next section.

\begin{figure*}[t]
  \centering
    \begin{subfigure}[b]{.45\textwidth}
      \includegraphics[width=\textwidth]{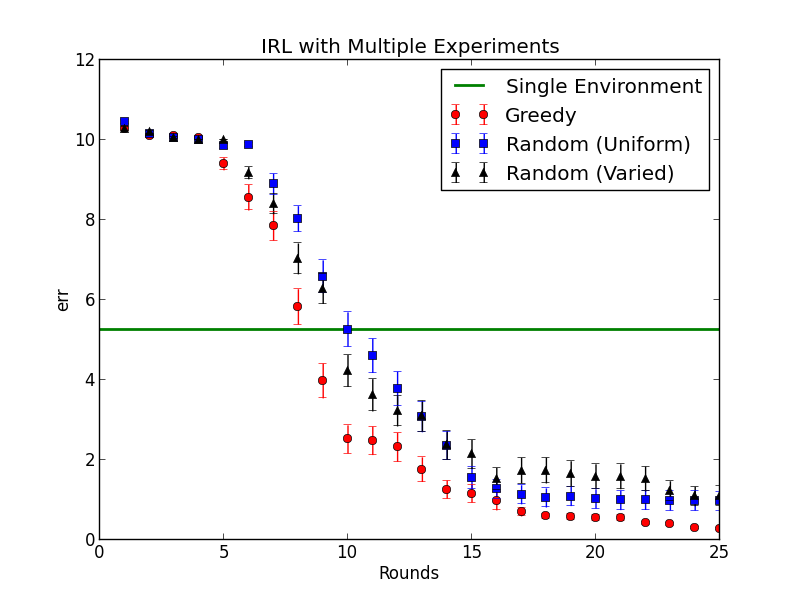}
      \caption{Policy Observations}
      \label{fig:lambda0}
    \end{subfigure}
    \begin{subfigure}[b]{.45\textwidth}
      \includegraphics[width=\textwidth]{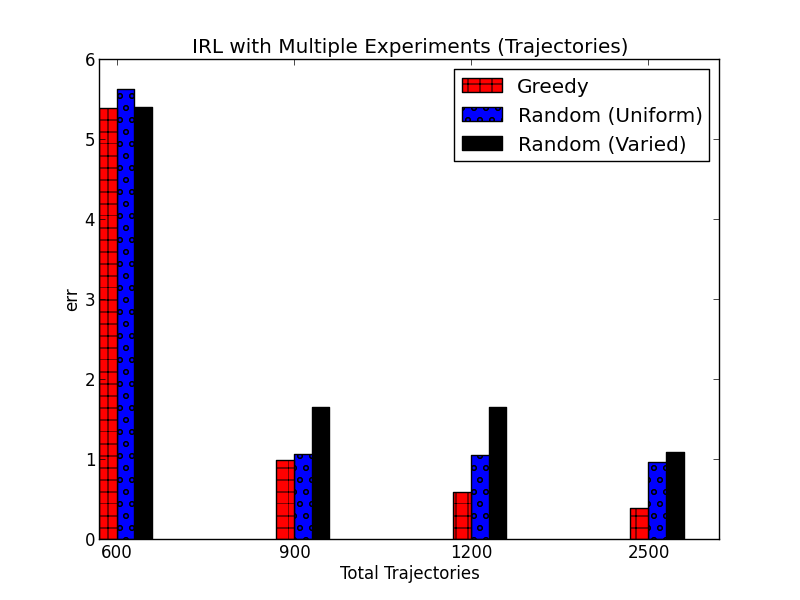}
      \caption{Trajectory Observations}
      \label{fig:traj}
    \end{subfigure}

  \caption{ Plot (a) displays $||\hat{R}||_{\infty}$ error for predicted vector $\hat{R}$ in the policy observation setting, with bars indicating standard error. Plot (b) displays the same in the trajectory setting. }
  \label{fig:results2}
\end{figure*}

\subsection{Adaptive Experimentation}
\label{sec:greedy}
Given a universe $\mathcal{U}$ of candidate environments, we now ask how to select a small number of environments from $\mathcal{U}$ so that the environments are maximally informative. We must first decide what we mean by ``informative.'' We propose that for a set of experiments $\mathcal{E}$ (either in the policy or trajectory setting), a natural objective is to minimize the mass of the resulting space of possible rewards $K(\mathcal{E})$ with respect to some measure (or distribution) $\mu$. Under the Lebesgue measure (or uniform distribution), this corresponds to the natural goal of reducing the volume of the $K(\mathcal{E})$ as much as possible. Thus we define: \begin{align*}\mathrm{Vol}_{\mu}(K(\mathcal{E})) &= \int_{\mathbb{R}^d} \mathbf{1}[R \in K(\mathcal{E})] \mathrm{d}\mu(R)\\
 				 &= \mathbb{P}_{R \sim \mu}\left[ R \in K(\mathcal{E}) \right]
\end{align*}

We will find it convenient to cast this as a maximization problem, and therefore also define $f(\mathcal{E}) = V - \mathrm{Vol}_{\mu}(K(\mathcal{E}))$, where $V$ is an upper bound on the volume of $[-R_{\max}, R_{\max}]^d$, and our goal to maximize $f(\mathcal{E})$. 

This objective has several desirable properties. First and foremost, by reducing the volume of $K(\mathcal{E})$ we eliminate the space of possible reward functions (i.e. experimental unidentifiability). Secondly, the repeated experimentation setting is fundamentally an active learning setting. We can think of the true, unknown, $R$ as a function that labels environments $E$ with either a corresponding policy $\pi$ or trajectory $T$. Thus, the volume operator corresponds to reducing the \emph{version space} of possible rewards. Furthermore, as we will see later in this section, the objective is a monotone submodular function, an assumption well-studied in the active learning literature \cite{guillory2010interactive,golovin2010adaptive}, allowing us to prove guarantees for a greedy algorithm. 

\begin{comment}
We argue that this objective has several desirable properties. First, we observe that when confronted with the perceived unidentifiability inherent in IRL, many pre-existing techniques for IRL rely on behavior-matching, rather than reward reconstruction \kareem{cite, cite, cite}. While this is useful for many applications, one of the major aims of this work is to resurrect the goal of pure identification in IRL as viable and not ill-posed.  the space of possible alternatives directly models this goal.

Secondly, the repeated experimentation setting is fundamentally an active learning setting. We can think of the true, unknown, $R$ as a function that labels environments $E$ with either a corresponding policy $\pi$ or trajectory $T$. Thus, the volume operator corresponds to reducing the \emph{version space} of possible rewards. Furthermore, as we will see later in this section, the objective is a monotone submodular function, which assumption is well-studied the active learning literature \kareem{cite}, allowing us to give provable guarantees for a simple greedy algorithm. 
\end{comment}

Finally, we will normally think of $\mu$ as being the Lebesgue measure, and $\mathrm{Vol}(\cdot)$ as volume in $d$-dimensional Euclidean space (or the uniform distribution on $[-R_{\max}, R_{min}]^d$). However, the choice of $\mu$ makes the objective quite general. For example, by making $\mu$ uniform on an $\epsilon$-net on $\mathbb{R}^d$, $\mathrm{Vol}$ corresponds to counting the number of rewards that are $\epsilon$-apart with respect to some metric. In many settings, $R$ naturally comes from some discrete space, such as the corners of the hypercube $\{0,1\}^d$. Again, this is readily modeled by the correct choice of $\mu$. In fact, $\mu$ can be thought of simply as any prior on $[-R_{\max}, R_{\max}]^d$. 

We are now ready to describe a simple algorithm that adaptively selects environments $E \in \mathcal{U}$, attempting to greedily maximize $f(\cdot)$, depicted as Algorithm \ref{alg:greedy}.

\begin{algorithm}[h]
\begin{algorithmic}[1]
\small \State {\bf Input} $B$
\State $\mathrm{i} := 1$
\State $\mathcal{E} := \emptyset$
\While {$i \leq B$}
\State $E_i := \underset{E}{\arg\max} \underset{R \in K(\mathcal{E})}{\min} \underset{\pi \in \mathrm{OPT}(E, R)}{\min} \small{f(\mathcal{E} \cup (E, \pi)) - f(\mathcal{E})}$ \label{line:max}
\State Observe policy $\pi^i$ for $E^i$. 
\State $\mathcal{E} := (\mathcal{E}, (E^i, \pi^i))$
\State $i := i + 1$
\EndWhile
\State \Return $\mathcal{E}$
\end{algorithmic}
\caption{Greedy Environment Selection}
\label{alg:greedy}
\end{algorithm}

In order to state a performance guarantee about Algorithm \ref{alg:greedy}, we will use the fact that $f(\cdot)$ is a submodular, non-decreasing, function on subsets of environment, observation pairs, $2^{\mathcal{\mathcal{U}} \times O}$, where $O$ is the set of possible observations.

\begin{lem}\label{lem:submod}
$f$ is a submodular, non-decreasing function. 
\end{lem}
\begin{proof}
Given a set $\mathcal{S}$ and component $s$, we use $\mathcal{S} + s$ to denote the union of the singleton set $\{s\}$ with $\mathcal{S}$. Let $O$ be the set of possible observations, so that $o$ is a trajectory in the trajectory setting, and a policy in the policy setting. Let $\mathcal{U}$ be the space of possible environments. 

Fix any $\hat{\mathcal{E}} \subset \mathcal{E} \subset 2^{\mathcal{U} \times O}$, and $(E,o) \not\in \mathcal{E}$. By definition of $K(\cdot)$, we have that $K(\mathcal{E} + (E, o)) = K(\mathcal{E}) \cap K(E, o)$ and $K(\mathcal{E}) \subset K(\hat{\mathcal{E}})$, and so: $f((\mathcal{E}, (E,o))) - f(\mathcal{E}) = \mathrm{Vol}(K(\mathcal{E})) - \mathrm{Vol}(K(\mathcal{E}, (E,o))) = \int_{\mathbb{R}^d} \mathbf{1}[R \in K(\mathcal{E}), R \not\in K(E,o)] \mathrm{d}\mu(R)$\\ $\leq \int_{\mathbb{R}^d} \mathbf{1}[R \in K(\hat{\mathcal{E}}), R \not\in K(E,o)] \mathrm{d}\mu(R) = f((\hat{\mathcal{E}}, (E,o))) - f(\hat{\mathcal{E}})$
This establishes submodularity of $f$. Since $\mathcal{E}$ is arbitrary and the right-hand-side of the second equality is non-zero, $f$ is also monotone.
\end{proof}

The performance of any algorithm is a function of how many experiments are attempted, and thus our analysis must take this into account. Let $\mathcal{A}_n$ be a deterministic algorithm that deploys at most $n$ experiments. $\mathcal{A}_n$ has a worst-case performance, which depends on the true reward $R$ and what policies were observed. We say a sequence of experiments $\mathcal{E} = ((E^1, o^1),...,(E^n, o^n))$ is consistent with $\mathcal{A}_n$ and $R$, if $\mathcal{A}_n$ chooses environment $E^{j+1}$ after observing the subsequence of experiments $((E^1, o^1),...,(E^j, o^j))$, and $o^j$ is either a trajectory or policy consistent with $(E,R)$. Denoting the set of consistent experiments $\mathcal{C}(\mathcal{A}_n, R)$, the best performance that any algorithm can \emph{guarantee} with $n$ experiments is: $\mathrm{OPT}_n = \max_{\mathcal{A}_n} \min_{R} \min_{\mathcal{E} \in \mathcal{C}(\mathcal{A}_n, R)} f(\mathcal{E})$

\begin{comment}
\begin{proof}
Fix any $\hat{\mathcal{E}} \subset \mathcal{E} \subset 2^{\mathcal{U} \times O}$, and $(E,o) \not\in \mathcal{E}$. By definition of $K(\cdot)$, we have that $K((\mathcal{E}, (E, o))) = K(\mathcal{E}) \cap K(E, o)$ and $K(\mathcal{E}) \subset K(\hat{\mathcal{E}})$, and so: 
$f((\mathcal{E}, (E,o))) - f(\mathcal{E}) = \mathrm{Vol}(K(\mathcal{E})) - \mathrm{Vol}(K(\mathcal{E}, (E,o))) = \int_{\mathbb{R}^d} \mathbf{1}[R \in K(\mathcal{E}), R \not\in K(E,o)] \mathrm{d}\mu(R) \leq \int_{\mathbb{R}^d} \mathbf{1}[R \in K(\hat{\mathcal{E}}), R \not\in K(E,o)] \mathrm{d}\mu(R) = f((\hat{\mathcal{E}}, (E,o))) - f(\hat{\mathcal{E}})$
%{\small \begin{align*}
%f((\mathcal{E}, (E,o))) &- f(\mathcal{E}) = \mathrm{Vol}(K(\mathcal{E})) - \mathrm{Vol}(K(\mathcal{E}, (E,o))) \\ &= \int_{\mathbb{R}^d} \mathbf{1}[R \in K(\mathcal{E}), R \not\in K(E,o)] \mathrm{d}\mu(R) \\&\leq \int_{\mathbb{R}^d} \mathbf{1}[R \in K(\hat{\mathcal{E}}), R \not\in K(E,o)] \mathrm{d}\mu(R) \\ &= f((\hat{\mathcal{E}}, (E,o))) - f(\hat{\mathcal{E}})\end{align*}}
This establishes submodularity of $f$. Since $\mathcal{E}$ is arbitrary and the first integral is non-zero, $f$ is also monotone. 
\end{proof}
\end{comment}

The submodularity of $f$, allows us to prove that for any $n$, the Greedy Environment Selection Algorithm\footnote{n.b. in the trajectory setting, one would replace the minimization over $\pi \in \mathrm{OPT}(E,R)$ in line \ref{line:max} of the algorithm, with a minimization over $T$ consistent with $\pi, \pi \in \mathrm{OPT}(E,R)$.}  needs slightly more than $n$ experiments (by a logarithmic factor) to attain $f(\mathcal{E}) \approx \mathrm{OPT}_n$.

\begin{thm}\label{thm:greedy}
$\mathcal{E}$ returned by the Greedy Environment Selection algorithm satisfies $f(\mathcal{E}) \geq \mathrm{OPT}_n - \epsilon$ when $B = n\ln(\mathrm{OPT}_n/\epsilon) \leq n \ln(V/\epsilon)$. 
\end{thm}

The proof of Theorem \ref{thm:greedy} uses many of the same techniques used by Guillory et. al (\citet{guillory2010interactive}), in their work on interactive set cover. For technical reasons, we cannot state our theorem directly as a corollary of these results, which assume a finite hypothesis class, whereas we have an infinite space of possible rewards. Nevertheless, these proofs are easily adapted to our setting, and the full proofs are given in the appendix. 

\begin{comment}
the key is deriving a version of Lemma 3 in \citet{guillory2010interactive}. We provide this as Lemma 2 in Section \ref{sec:proofs} for completeness. 
\end{comment}

Finally we note that Line (\ref{line:max}) is not computable exactly without parametric assumptions on the class of environments or space of rewards. In practice, and as we will describe in the next section, we approximate the exact maximization by sampling environments and rewards from $K(\mathcal{E})$, and optimizing on the sampled sets.

\section{Experimental Analysis}

\begin{comment}
\begin{figure}[h]
    \centering
    \begin{subfigure}[b]{0.3\textwidth}
  \includegraphics[width=\textwidth]{./mazequad.png}
  \end{subfigure}
  \caption{{\small (a) An agent's policy in a fixed environment. An agent can move in one of four directions or can stay at a location (represented by the black circle). Thick purple lines represent impassible walls. (b) Does space $y$ have a reward associated with it? (c) An experiment revealing that $y$ is in fact rewarding. (d) An experiment revealing that if getting to $y$ requires 11 steps or more, and getting to $x$ requires 4 or fewer, the agent prefers $y$.}}
  \label{fig:mazes}
\end{figure}
\end{comment}

%Plots (c) and (d) show the effect of the choice of $\lambda$ in the policy oracle setting.

We now deploy the techniques discussed in a setting, demonstrating that maximizing $f(\cdot)$ is indeed effective for identifying $R$. We imagine that we have an agent that will be dropped into a grid world. The experimenter would like to infer the agent's reward for each space in the grid. We imagine that the experimenter has the power to construct walls in the agent's environment, and so we will alternatively refer to an environment as a \emph{maze}. To motivate the value of repeated experimentation, recall Figure \ref{fig:mazes}. 

This is a restricted environment for the learner. The learner cannot, for example, make it so that an action causes the agent to travel from a bottom corner of the maze to a top corner. However, the learner can modify the dynamics of the environment in so far as it can construct maze walls. 

\begin{comment}
\kareem{move this discussion?}
In Figure \ref{fig:mazes}a, we witness the agent's policy in a fixed environment. The agent makes its way along the shortest path to position $x$, then stays there (indicated by the circle). Using IRL techniques prior to this work, we might infer that the shaded space labeled $x$ in the bottom left corner is a rewarding state, and might assign $x$ with reward $R_{\max}$, and the remaining spaces with reward $0$. We now invite the reader to more carefully consider the reward of $y$ in Figure \ref{fig:mazes}b. Consider two hypotheses. The first, $H_0$, is our previous hypothesis that $x$ gives a non-zero reward, while $y$ gives a reward of $0$. The second, $H_1$, is that $y$ gives a non-zero reward less than that of $x$, chosen so that among policies that make their way to $x$, the agent will always prefer ones that visit $y$. In the standard IRL setting, distinguishing between these two hypotheses is impossible; the experimenter simply cannot know whether the choice to visit $y$ was arbitrary, or a deliberate attempt to collect some small, but significant, reward \emph{en route} to the larger reward at $x$. However, with multiple experiments, this unidentifiability vanishes. We might first conduct the experiment in Figure \ref{fig:mazes}c, witnessing that $y$ is in fact a rewarding state. By introducing more convoluted routes to $x$ (Figure \ref{fig:mazes}d), we may even infer the exact break-even point for some discount factor, when the agent would prefer to stay at $y$, invalidating $H_0$. 
\end{comment}

We evaluate Algorithm \ref{alg:greedy} on grids of size $10 \times 10$. An agent's reward is given by a vector $R \in \mathbb{R}^{100}$, with $|| R ||_{\infty} \leq R_{\max}$, where $R_{\max}$ is taken to be $10$ in all that follows. In each simulation %\footnote{Due to the nature of this work, we reserve the phrase ``experiment'' to refer to the learner's choice of environment, and subsequent observation. We will exclusively use the phrase ``simulation'' to refer to our -- the authors' -- validation of the techniques discussed in this paper.}
we randomly assign some state in $R$ to have reward $R_{\max}$, and assign $5$ states to have reward $1$.\footnote{For motivation, one might think of the agent as being a mouse, with these rewards corresponding to food pellets or various shiny objects in a mouse's cage.} 
The remaining states give reward $0$. The agent's discount rate is taken to be $0.8$. The goal of the learner is not just to determine which states are rewarding, but to further determine that the latter states yield $1/10$ the reward of the former.

In Figure \ref{fig:lambda0}, we display our main experimental results for four different algorithms in the policy observation setting, and in Figure \ref{fig:traj} for the trajectory setting. Error represents $||R - \hat{R}||_\infty$, where $\hat{R}$ is an algorithm's prediction, with error bars representing standard error over $20$ simulations. 

In Figure \ref{fig:lambda0}, the horizontal line displays the best results we achieved without repeated experimentation. If the learner only selects a single environment $E$, observing policy $\pi$, it is stuck with whatever experimental unidentifiability exists in $K(E,\pi)$. In such a scenario, we can select a $K(E,\pi)$ according to a classic IRL heuristic, given by LP (\ref{eqn:classic}) in Section \ref{sec:heuristics}, for some choice of $\lambda$ in LP (\ref{eqn:classic}). Since the performance of this method depends both on which environment is used, and the choice of $\lambda$, we randomly generated $100$ different environments, and for each of those environments selected $\lambda \in \{0.05,.1,.5,1,5,6,7,8,9,10\}$. We then evaluated each of these single-environment approaches with $20$ simulations, the best error among these $1300$ different single-environment algorithms is displayed by the horizontal line. Immediately we see that the experimental unidentifiability  from using a single environment makes it difficult to distinguish the actual reward function, with $\mathrm{err}$ for the best choice of $E$ and $\lambda$ greater than $5$. 

\begin{comment}
\begin{figure}[h]
  \centering
    \begin{subfigure}[b]{.35\textwidth}
      \includegraphics[width=\textwidth]{./results1b.png}
    \end{subfigure}
  \caption{}
  \label{fig:results1}
\end{figure}
\end{comment}

The remaining algorithms --- which we will describe in greater detail below --- conduct repeated experimentation. Each of these algorithms uses a different rule to select a new environment on each round. Given the sequence of (environment, policy) pairs $\mathcal{E}$ generated by each of these algorithms, we solve the LP (\ref{eq:rule}) on $K(\mathcal{E})$ at the end of each round. This is done with the same choice of $\lambda (=0.5)$ for each of the algorithms. 

Besides the $\mathrm{Greedy}$ algorithm of the previous section, we implement two other algorithms, which conduct repeated experiments, but do so non-adaptively. $\mathrm{RandUniform}$, in each round, selects a maze uniformly at random from the space of possible mazes (each wall is present with probability $0.5$). Note that $\mathrm{RandUniform}$ will tend to select mazes where roughly half of the walls are present. Thus, we also consider $\mathrm{RandVaried}$ which, in each round, selects a maze from a different distribution $\mathcal{D}$. Mazes drawn from $\mathcal{D}$ are generated by a two-step process. First, for each row $r$ and column $c$, we select numbers $d_r,d_c$ i.i.d. from the uniform distribution on $[0,1]$. Then each wall along row $r$ (column $c$ respectively) is created with probability $d_r$ ($d_c$ respectively). Although the probability any particular wall is present is still $0.5$, the correlations in $\mathcal{D}$ creates more variable mazes (e.g. allowing an entire row to be sparsely populated with walls). 

We implement Algorithm \ref{alg:greedy}, $\mathrm{Greedy}$, of the previous section, by approximating the maximization in Line \ref{line:max} in Algorithm \ref{alg:greedy}. This approximation is done by sampling $10$ environments from $\mathcal{D}$, the same distribution used by $\mathrm{RandVaried}$. In the policy observation setting, $1000$ samples are first drawn from the consistent set $K(\mathcal{E})$ using a hit-and-run sampler \cite{lovasz1999hit}, which is an MCMC method for uniformly sampling high-dimensional convex sets in polynomial time. These same samples are also used to estimate the volume $f(\cdot)$. In the trajectory setting, we first sample trajectories $T$ on an environment $E$, then we use $K(E,\hat{\pi})$ for an arbitrary $\hat{\pi}, \hat{\pi} \equiv_{\mathcal{D}(T)} \pi_T$, as a proxy for $K(E,T)$.

Examining the results, we see that $\mathrm{Greedy}$ converges significantly quicker than either of the non-adaptive approaches. After $25$ rounds of experimentation in the policy observation setting, $\mathrm{Greedy}$ attains error $0.2687 (\pm 0.0302)$, while the best non-adaptive approach attains $0.9691 (\pm 0.24310)$. $\mathrm{Greedy}$ only requires $16$ rounds to reach a similar error of $0.9678 (\pm 0.0701)$. We note further that the performance of $\mathrm{Greedy}$ seems to continue to improve, while the non-adaptive approaches appear to stagnate. This could be due to the fact that after a certain number of rounds, the non-adaptive approaches have received all the information available from the environments typically sampled from their distributions. In order to make progress they must receive new information, in contrast to $\mathrm{Greedy}$, which is designed to actively select the environments that will do just that. 

Finally, $\mathrm{Greedy}$ runs by selecting a sequence of environments, resulting in observations $\mathcal{E}$. It then selects $R$ from $K(\mathcal{E})$ using LP (\ref{eq:rule}). Thus, the regularization parameter $\lambda$ in LP (\ref{eq:rule}) is a free parameter for $\mathrm{Greedy}$ that we took to be equal to $0.5$ for results (Figure \ref{fig:lambda0}). We conclude by experimentally analyzing the sensitivity of $\mathrm{Greedy}$ to the choice of this parameter, as well as of $\mathrm{RandUniform}$, and $\mathrm{RandVaried}$, which also select $R$ according to LP (\ref{eq:rule}). As $\lambda$ is increased, eventually the LP over-regularizes, and is optimized taking $\mathrm{R} = \vec{0}$. In our setting, once $\lambda \approx 1$ this begins to occur, and we begin to see pathological behavior (Figure \ref{fig:lambda1}). This problem occurs in standard IRL, and one approach (\cite{ng2000algorithms}) is to select a large lambda before this transition, hence our choice of $\lambda = 0.5$. However, even for significantly smaller $\lambda$, the results are qualitatively similar (Figure \ref{fig:lambda2}) to those in Figure \ref{fig:lambda0}. We find that as long as $\lambda$ is not too large, the results are not sensitive to the choice of $\lambda$. 
\begin{figure}[h!]
    \begin{subfigure}[b]{.23\textwidth}
      \includegraphics[width=\textwidth]{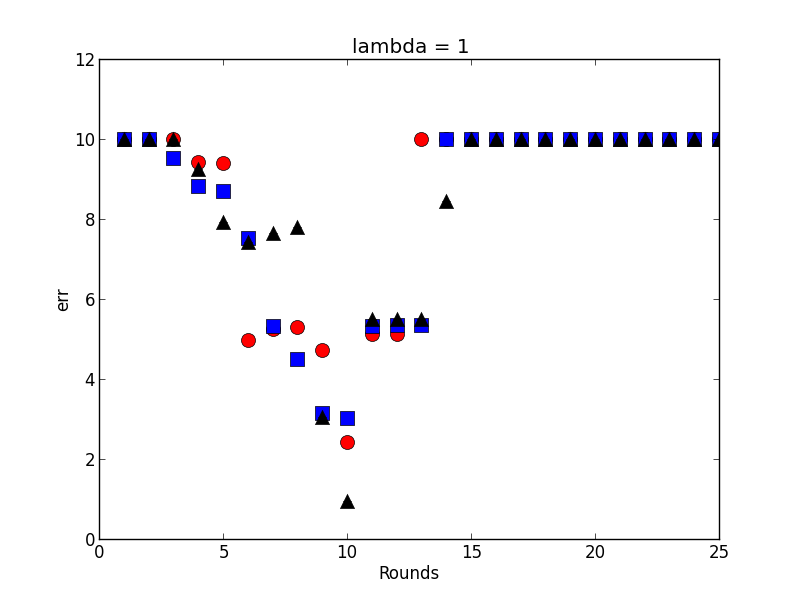}
      \caption{$\lambda = 1$}
      \label{fig:lambda1}
    \end{subfigure}
    \begin{subfigure}[b]{.23\textwidth}
      \includegraphics[width=\textwidth]{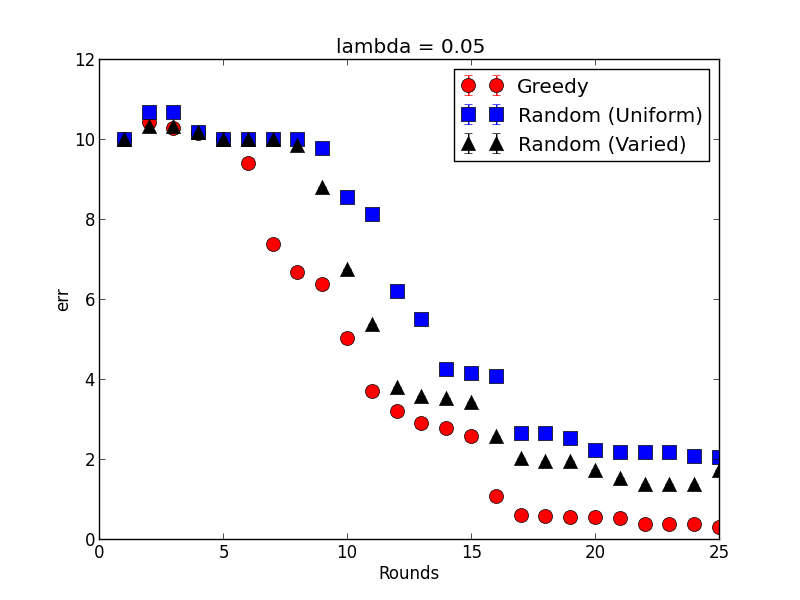}
      \caption{$\lambda = 0.05$}
      \label{fig:lambda2}
    \end{subfigure}
    \caption{Result for all repeated experimentation algorithms using large and small regularization parameter $\lambda$.}
\end{figure}
\section{Conclusions}
We provide a number of contributions in this work. First, we separate the causes of unidentifiability in IRL problems into two classes: representational, and experimental.  We argue that representational unidentifiability is superficial, leading us to redefine the problem of identification in IRL according to Definition \ref{defn:ident}. While previous work does not distinguish between these two classes, we demonstrate that, by doing so, algorithms can be designed to eliminate experimental unidentifiability while providing formal guarantees. 

Along the way, we derive a new model for IRL where the learner can observe behavior in multiple environments, a model which we believe is interesting in its own right, but also is key to eliminating experimental unidentifiability. We give an algorithm for a very powerful learner who can observe agent behavior in any environment, and show that the algorithm $\epsilon$-identifies an agent reward defined on $d$ states, while observing behavior on only $O(\log(d/\epsilon))$ environments. We then weaken this learner to model more realistic settings where the learner might be restricted in the types of environments it may choose, and where it may only be able to elicit a small number of demonstrations from the agent. We derive a simple adaptive greedy algorithm which will select a nearly optimal (with respect to reducing the volume of possible reward function) set of environments. The value of the solution found by this greedy algorithm will be a comparable to the optimal algorithm which uses a logarithmic factor fewer number of experiments. 

Finally, we implement the algorithm in a simple maze environment that nevertheless demonstrates the value of eliminating experimental unidentifiability, significantly outperforming methods that attempt to perform IRL from a single environment.

\bibliographystyle{abbrv}
\bibliography{inverse.bib}

\appendix

\section{Proof of Theorem 3}

\begin{thm}
For any $R,R' \in \mathbb{R}^d$, $R \equiv R'$ if and only if they have the same canonicalized representation. 
\end{thm}
\begin{proof}
By definition, the canonicalized representation of any reward function is attained by scaling and translation. Therefore, by Theorem \ref{thm:equiv}, if $R$ and $R'$ are both canonicalized as $R_c$, we have that $R \equiv R_c$ and $R' \equiv R_c$, and therefore $R \equiv R'$. 

In the other direction, suppose $R$ and $R'$ are canonicalized to $R_c$ and $R_c'$ respectively, where $R_c \not= R_c'$. Again, by Theorem \ref{thm:equiv}, we have that $R \equiv R_c$ and $R' \equiv R_c'$. Thus, to prove the theorem, it is sufficient to argue that $R_c$ and $R_c'$ are not behaviorally equivalent. 

If one of $R_c,R_c'$ is $\vec{0}$ and the other is not, then it is straightforward to show that they are not behaviorally equivalent. Thus, we focus on the case where both $R_c$ and $R'_c$ are not $\vec{0}$. We consider three cases. 

First, suppose that $R_c \not= R_c'$ because they have different minimally-rewarding states. Without loss of generality suppose that there is some $s_0$ with $R_c(s_0) = 0$ but $R_c'(s_0) > 0$. Furthermore, let $s'_0$ be any state such that $R_c'(s'_0) = 0$. Consider an environment $E$ with two actions $a$ and $a'$. Action $a$ deterministically transitions to state $s_0$ from any other state, while action $a'$ determininstically transitions to state $s_0'$ from any other state. Let $\pi_a$ be the policy that always takes action $a$. $\mathrm{OPT}(E,R'_c) = \{\pi_a\}$. However, if $\pi_a \in \mathrm{OPT}(E, R_c)$, this means that $R_c(s_0') = 0$, and therefore all policies are in $\mathrm{OPT}(E,R_c)$. Thus, $\mathrm{OPT}(E,R_c) \not= \mathrm{OPT}(E,R'_c)$, and $R_c,R'_c$ are not behaviorally equivalent. 

Next, suppose that $R_c \not= R_c'$ because they have different maximally-rewarding states. Analagously to the previous case, suppose without loss of generality there is some $s_0$ with $R_c(s_0) = 1$ by $R'_c(s_0) < 0$, and let $s'_0$ be any state such that $R_c'(s_0') = 1$ (which exists since $R_c' \not= \vec{0}$). Define the environment $E$ in the same way as the previous case. This time, $\mathrm{OPT}(E,R'_c) = \{\pi_{a'}\}$, while $\mathrm{OPT}(E,R_c) \not= \{\pi_{a'}\}$. 

Finally, suppose that $R_c$ and $R_c'$ share the same maximally and minimally rewarding states, but there exists some $s$ such that $R_c(s) \not= R_c'(s)$. Let $s_0$ be any state such that $R_c(s_0) = R_c'(s_0) = 0$ and let $s_1$ be any state such that $R_c(s_1) = R'_c(s_1) = 1$. Without loss of generality suppose that $R_c(s) < R_c'(s)$. Let $E$ be the environment with two actions $a$ and $a_p$. Let $p$ be any real number $0 \leq R_c(s) < p < R_c'(s) \leq 1$. From every state, action $a_p$ transitions to state $s_1$ with probabiity $p$ and to state $s_0$ with the remaining probability. From every state action $a$ transtions to state $s$ deterministically. The reward for taking action $a_p$ in any state under either reward function is $p$, while action $a$ gives a reward of $R_c(s) < p$ under $R_c$ and $R'_c(s) > p$ under $R_c'$. Thus, $\mathrm{OPT}(E,R_c) = \{\pi_{a_p}\} \not= \{\pi_{a}\} = \mathrm{OPT}(E,R'_c)$, concluding the proof.
\end{proof}

\section{Proof of Greedy's Performance}

Given a set $\mathcal{S}$ and component $s$, we use $\mathcal{S} + s$ to denote the union of the singleton set $\{s\}$ with $\mathcal{S}$. We begin by redefining:

$$\mathrm{Vol}_{\mu}(K(\mathcal{E})) = \int_{\mathbb{R}^d} \mathbf{1}[R \in K(\mathcal{E})] \mathrm{d}\mu(R)$$ 

$$f(\mathcal{E}) = V - \mathrm{Vol}_{\mu}(K(\mathcal{E}))$$ where $V$ is an upper bound $\mathrm{Vol}_{\mu}([-R_{\max}, R_{\max}]^d)$.

Let $O$ be the set of possible observations, so that $o$ is a trajectory in the trajectory setting, and a policy in the policy setting. Let $\mathcal{U}$ be the space of possible environments. WWe first establish that $f$ is indeed submodular. 

\begin{lem}
$f$ is a submodular, non-decreasing function on $2^{\mathcal{U} \times O}$. 
\end{lem}
\begin{proof}
Fix any $\hat{\mathcal{E}} \subset \mathcal{E} \subset 2^{\mathcal{U} \times O}$, and $(E,o) \not\in \mathcal{E}$. By definition of $K(\cdot)$, we have that $K(\mathcal{E} + (E, o)) = K(\mathcal{E}) \cap K(E, o)$ and $K(\mathcal{E}) \subset K(\hat{\mathcal{E}})$, and so: \begin{align*}
f((\mathcal{E}, (E,o))) &- f(\mathcal{E}) = \mathrm{Vol}(K(\mathcal{E})) - \mathrm{Vol}(K(\mathcal{E}, (E,o))) \\ &= \int_{\mathbb{R}^d} \mathbf{1}[R \in K(\mathcal{E}), R \not\in K(E,o)] \mathrm{d}\mu(R) \\&\leq \int_{\mathbb{R}^d} \mathbf{1}[R \in K(\hat{\mathcal{E}}), R \not\in K(E,o)] \mathrm{d}\mu(R) \\ &= f((\hat{\mathcal{E}}, (E,o))) - f(\hat{\mathcal{E}})\end{align*}
This establishes submodularity of $f$. Since $\mathcal{E}$ is arbitrary and the right-hand-side of the second equality is non-zero, $f$ is also monotone. 
\end{proof}

Let $\mathcal{T} = \{ T : \mathcal{U} \rightarrow O\}$ denote the set of functions mapping environments to observations. For any $T \in \mathcal{T}$ and $S \subset \mathcal{U}$, overload $T$, so that $T(S) = \cup_{E \in S} (E, T(E))$. 

Now suppose that environments where labeled according to some $T \in \mathcal{T}$, and consider an algorithm which knowing $T$, selects the fewest number of environments $S$, so that $f(T(S)) \geq \alpha$. Given such an algorithm, we can now define the General Identification Cost, which identifies the worst-possible labelling strategy in $\mathcal{T}$. In particular:

$$\mathrm{GIC}_{\alpha} = \max_{T \in \mathcal{T}} \min_{S \subset \mathcal{U} : f(T(S)) \geq \alpha} |S|$$

Recall the definition from the main body:
$$\mathrm{OPT}_n = \max_{\mathcal{A}_n} \min_{R} \min_{\mathcal{E} \in \mathcal{C}(\mathcal{A}_n, R)} f(\mathcal{E})$$

This is the largest that an algorithm can guarantee to make $f(\cdot)$ with $n$ environments, when environments are consistently labeled by some $R$. Let $A^*$ be the algorithm satisfying the $\max$.  

\begin{lem}\label{lem1}
$\mathrm{GIC}_{\mathrm{OPT}_n} \leq n$
\end{lem}
\begin{proof}
Fix any $T \in \mathcal{T}$. Consider two cases. First suppose that there exists some $S \subset \mathcal{U}$ such that $|S| \leq n$, but $T(S)$ is inconsistent with the labeling of any $R$. By defintion of $f$, $f(T(S)) = V \geq \mathrm{OPT}_n$, and since $|S| \leq n$, the Lemma is proven. 

Otherwise, it must be that all $S \subset \mathcal{U}$, $|S| \leq n$, $T(S)$ is consistent with the labeling of some $R$. By definition of $\mathrm{OPT}_n$, running $A^*$ against the labels provided by $T$ is guaranteed to result in a sequence of environments $S^*$, $|S^*| \leq n$, satisfying $f(T(S^*)) \geq \mathrm{OPT_n}$. $S^*$ is a witness that $\min_{S \subset \mathcal{U} : f(T(S)) \geq \mathrm{OPT}_n} |S|$ is at most $n$.  
\end{proof}

Given an environment $E$ and true reward $R$, let $O(E,R)$ denote the set of possible observations (in either the policy or trajectory setting). 

\begin{lem}\label{lem2}
For any $\mathcal{E}$, such that $f(\mathcal{E}) \leq \mathrm{OPT}_n$, there exists an environment $E$ such that:

$$\min_{R \in K(\mathcal{E})} \min_{o \in O(E,R)} f(\mathcal{E} + (E,o)) - f(\mathcal{E}) \geq (\mathrm{OPT}_n - f(\mathcal{E}))/\mathrm{GIC}_{\mathrm{OPT}_n}$$
\end{lem}
\begin{proof}
Suppose not. Then for every environment $E$, there exists some $R \in K(\mathcal{E})$ and $o \in O(E,R)$ such that:
$$f(\mathcal{E} + (E,o)) - f(\mathcal{E}) < (\mathrm{OPT}_n - f(\mathcal{E}))/\mathrm{GIC}_{\mathrm{OPT}_n}$$

Now let $T' \in \mathcal{T}$ be defined so that:

\begin{align}
T'(E) &\triangleq \arg\min_{R, o \in O(E,R)} f(\mathcal{E} + (E,o)) - f(\mathcal{E}) \\ 
      &< (\mathrm{OPT}_n - f(\mathcal{E}))/\mathrm{GIC}_{\mathrm{OPT}_n}\label{eqn:x}
\end{align}

By definition of $\mathrm{GIC}$, we have that:

$$ \min_{S \subset \mathcal{U} : f(T'(S)) \geq \mathrm{OPT}_n} |S| \leq \mathrm{GIC}_{\mathrm{OPT}_n}$$

So there exists a set of environments $S$, with $|S| \leq \mathrm{GIC}_{\mathrm{OPT}_n}$, such that $f(T'(S)) \geq \mathrm{OPT}_n$, and by monotonicity of $f$, we know that $f(T'(S) \cup \mathcal{E}) \geq \mathrm{OPT}_n$. Let $\gamma = |S|/\mathrm{GIC}_{\mathrm{OPT}_n}$.

However, despite $f(T'(S) \cup \mathcal{E}) \geq \mathrm{OPT}_n$, repeatedly applying the submodularity of $f$, then applying equation (\ref{eqn:x}) implies:

\begin{align*}
f(T'(S) \cup \mathcal{E}) &\leq f(\mathcal{E}) + \sum_{E \in S} (f(\mathcal{E} + (E,T'(E))) - f(\mathcal{E})\\
                          &< f(\mathcal{E}) + |S|(\mathrm{OPT}_n - f(\mathcal{E}))/\mathrm{GIC}_{\mathrm{OPT}_n}\\
                          &= f(\mathcal{E}) + \gamma(\mathrm{OPT}_n - f(\mathcal{E})) \\
                          &= (1-\gamma)f(\mathcal{E}) + \gamma \mathrm{OPT}_n \leq \mathrm{OPT}_n
\end{align*}
This establishes a contradiction.
\end{proof}

We can now prove the main theorem:
\begin{thm}
$\mathcal{E}$ returned by the Greedy Environment Selection algorithm satisfies $f(\mathcal{E}) \geq \mathrm{OPT}_n - \epsilon$ when $B = n \ln(\mathrm{OPT_n/\epsilon})$. 
\end{thm}
\begin{proof}
Let $\mathcal{E}_i$ denote the subsequence consisting of the first $i$ environment, observation pairs. If for some $i \leq B$, $f(\mathcal{E}_i) \geq \mathrm{OPT}_n$, then there is nothing to prove. Otherwise, by applying Lemma \ref{lem1}, and the definition of the algorithm, we know that:

$$f(\mathcal{E}_i) - f(\mathcal{E}_{i-1}) \geq (\mathrm{OPT}_n - f(\mathcal{E}_{i-1}))/\mathrm{GIC}_{\mathrm{OPT}_n}$$
 which implies:
$$\mathrm{OPT}_n - f(\mathcal{E}_i) \leq (\mathrm{OPT}_n - f(\mathcal{E}_{i-1}))(1 - 1/\mathrm{GIC}_{\mathrm{OPT}_n})$$

Using the fact $1-x \leq e^{-x}$, we can conclude:

$$\mathrm{OPT}_n - f(\mathcal{E}) \leq \mathrm{OPT}_n \exp(-B/\mathrm{GIC}_{\mathrm{OPT}_n}))$$

Applying Lemma \ref{lem2} and substituting $B$ completes the proof. 

\end{proof}

\end{document}